%% file: main.tex
\documentclass[letterpaper]{article} 
\usepackage{aaai24}  
\usepackage{times}  
\usepackage{helvet}  
\usepackage{courier}  
\usepackage[hyphens]{url}  
\usepackage{graphicx} 
\usepackage{natbib}  
\usepackage{caption} 
\usepackage{algorithm}
\usepackage[noend]{algorithmic}

\usepackage{multirow}
\usepackage{subcaption}
\usepackage{newfloat}
\usepackage{listings}
\DeclareCaptionStyle{ruled}{labelfont=normalfont,labelsep=colon,strut=off} 
\floatstyle{ruled}
\newfloat{listing}{tb}{lst}{}
\floatname{listing}{Listing}
\usepackage{amsmath}
\usepackage{amssymb}
\usepackage{amsthm}
\usepackage{mathtools}
\usepackage{bm}
\usepackage{booktabs}

\newtheorem{theorem}{Theorem}
\newtheorem{lemma}
{Lemma}
\newtheorem{assumption}{Assumption}
\newtheorem{corollary}{Corollary}
\newtheorem{remark}{Remark}

\newcommand{\la}{\langle\,}
\newcommand{\ra}{\,\rangle}
\newcommand{\argmin}{\mathop{\rm argmin}}
\newcommand{\argmax}{\mathop{\rm argmax}}

\newcommand{\E}{\mathbb{E}}
\newcommand{\R}{\mathbb{R}}

\newcommand{\cA}{\mathcal{A}}

\newcommand{\cS}{\mathcal{S}}

\newcommand{\cM}{\mathcal{M}}

\newcommand{\bphi}{\bm{\phi}}

\newcommand{\sind}[3]{{#1}^{#2}_{#3}}

\newcommand{\iprod}[2]{\langle #1, #2 \rangle}

\newcommand{\abs}[1]{\left|{#1}\right|}
\newcommand{\norm}[1]{\left\|{#1}\right\|_2}
\newcommand{\onlynorm}[1]{\left\|{#1}\right\|}

\newcommand{\algo}{\textsf{LoBiSaRL}}

\usepackage{todonotes}

\title{Long-Term Safe Reinforcement Learning with Binary Feedback}
\author {
    Akifumi Wachi\textsuperscript{\rm 1},
    Wataru Hashimoto\textsuperscript{\rm 2},
    Kazumune Hashimoto\textsuperscript{\rm 2}
}
\affiliations {
    \textsuperscript{\rm 1} LINE Corporation \\
    \textsuperscript{\rm 2} Osaka University \\
    akifumi.wachi@linecorp.com, hashimoto@is.eei.eng.osaka-u.ac.jp, hashimoto@eei.eng.osaka-u.ac.jp
}

\usepackage{bibentry}

\begin{document}

\maketitle

\begin{abstract}
Safety is an indispensable requirement for applying reinforcement learning (RL) to real problems.
Although there has been a surge of safe RL algorithms proposed in recent years, most existing work typically 1) relies on receiving numeric safety feedback; 2) does not guarantee safety during the learning process; 3) limits the problem to a priori known, deterministic transition dynamics; and/or 4) assume the existence of a known safe policy for any states.
Addressing the issues mentioned above, we thus propose Long-term Binary-feedback Safe RL (\algo), a safe RL algorithm for constrained Markov decision processes (CMDPs) with binary safety feedback and an unknown, stochastic state transition function.
\algo~optimizes a policy to maximize rewards while guaranteeing a long-term safety that an agent executes only safe state-action pairs throughout each episode with high probability.
Specifically, \algo~models the binary safety function via a generalized linear model (GLM) and conservatively takes only a safe action at every time step while inferring its effect on future safety under proper assumptions.
Our theoretical results show that \algo~guarantees the long-term safety constraint, with high probability.
Finally, our empirical results demonstrate that our algorithm is safer than existing methods without significantly compromising performance in terms of reward.
\end{abstract}

\section{Introduction}
\label{sec:introduction}

Safe reinforcement learning (RL) is a promising paradigm for applying RL algorithms to real-world applications \cite{garcia2015comprehensive}.
Safe RL is beneficial in safety-critical decision-making problems, such as autonomous driving, healthcare, and robotics, where safety requirements must be incorporated to prevent RL policies from posing risks to humans or objects \cite{dulac2021challenges}.
As a result, safe RL has received significant attention in recent years as a crucial issue of RL during both the learning and execution phases \cite{amodei2016concrete}.

Safe RL is typically formulated as \textit{constrained} policy optimization problems where the expected cumulative reward is maximized while guaranteeing or encouraging the satisfaction of safety constraints, which are modeled as constrained Markov decision processes (CMDPs, \citet{altman1999constrained}).
While there are various types of constraint representations, most of the existing studies formulated constraints using either expected cumulative safety-cost~\cite{altman1999constrained} or conditional value at risk (CVaR, \citet{rockafellar2000optimization}); thus, satisfying safety constraints almost surely or with high probability received less attention to date.
Imagine highly safety-critical applications (e.g., autonomous driving, healthcare, robotics) where even a single constraint violation may result in catastrophic failure.
In such cases, RL agents need to ensure safety at every time step at least with high probability; thus, constraint satisfaction “on average” does not fit the purpose due to a large number of unsafe actions during the learning process~\cite{stooke2020responsive}.

\begin{figure}[t]
    \centering
    \includegraphics[width=80mm]{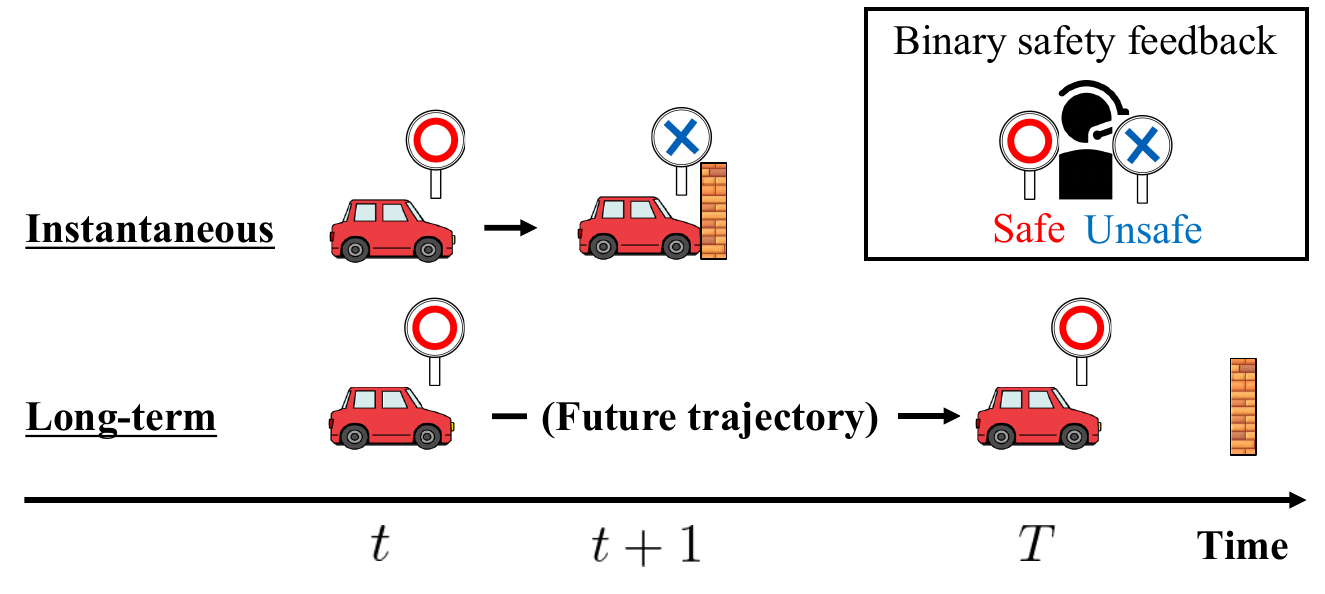}
    \caption{Even if safety is guaranteed at time $t$ based on the instantaneous evaluation, safe behavior may not exist a few steps ahead. This paper requires an agent to guarantee long-term safety (i.e., constraint satisfaction from the time the current time step $t$ to the terminal time step $T$) in CMDPs with stochastic state transition and binary safety feedback.}
    \label{fig:concept}
\end{figure}

\begin{table*}[t]
\centering
\begin{small}
\begin{tabular}{lccccc}
\toprule
& \multicolumn{2}{c}{State transition} & \multirow{2}{*}[-2pt]{Safety} & \multirow{2}{*}[-2pt]{Additional assumption(s)} \\
\cmidrule(lr){2-3}
& Known & D/S & & \\
\midrule
\citet{wachi2020safe} & Yes & D & GP & - \\
\citet{amani2021safe} & Linear & S & Linear & Known safe policy \\
\citet{wachi2021safe} & Yes & D & GLM & - \\
\citet{bennett2023provable} & No & S & GLM & Known safe policy \\
\algo~(Ours) & No & S & GLM & Lipschitz continuity \& conservative policy \\
\bottomrule
\end{tabular}
\end{small}
\caption{Comparison among existing work regarding their assumptions on a state transition, safety function, and others. In the above table, D means deterministic state transition, and S means stochastic state transition.}
\label{tab:problem}
\end{table*}

Several previous work on safe RL aimed to guarantee safety at every time step with high probability, even during the learning process.
Unfortunately, however, existing work has room for improvement.
First, most of the previous work \cite{wachi2020safe, amani2021safe, roderick2021provably, wachi2023safe} assumes numeric safety feedback.
In many cases, however, the safety feedback can only take binary values indicating whether a state-action pair is safe or unsafe, which is particularly true when feedback comes from humans.
As existing studies on safe RL with binary safety feedback, \citet{wachi2021safe} modeled the safety function via a generalized linear model (GLM) while they assume known and deterministic state transition function.
Thus, this previous work cannot deal with general RL problems with unknown stochastic state transition functions.
Also, \citet{bennett2023provable} addressed safe RL problems with binary safety feedback and unknown stochastic state transition under the assumption that a known safe action always exists for any state.
This assumption is not valid in many safety-critical applications.
For example, even an F1 driver cannot take a safe action if a vehicle traveling at $100$ km/h is $1$ meter ahead of a brick wall; thus, to avoid such situations, we need to consider ``long-term'' future safety under more reasonable assumptions, as shown in Figure~\ref{fig:concept}.

\paragraph{Contributions.}
We propose an algorithm called Long-term Binary-feedback Safe RL, \algo.
This algorithm enables us to solve safe RL problems with binary feedback and unknown, stochastic state transition while guaranteeing the satisfaction of long-term safety constraints.
\algo~guarantees safety by modeling the binary safety function via a GLM and then pessimistically estimating the future safety function values.
Our theoretical analysis shows that future safety can be pessimistically characterized by 1)~inevitable randomness due to the stochastic state transition and 2) divergence between the current policy and a reference policy to stabilize the state.
Based on this theoretical result, we optimize the policy to maximize the expected cumulative reward while guaranteeing long-term safety.
Finally, we empirically demonstrate the effectiveness of the \algo~compared with several baselines.

\section{Related Work}
\label{sec:related}

\paragraph{Safe RL.}

In typical safe RL problems, an agent must maximize the expected cumulative reward while ensuring that the expected cumulative cost is less than a threshold.
There have been a number of algorithms for solving this type of safe RL problem, as represented by constrained policy optimization \cite{achiam2017constrained}, reward constrained policy optimization \cite{tessler2018reward}, Lagrangian-based actor-critic \cite{chow2017risk}, primal-dual policy optimization~\cite{pmlr-v119-yang20h}.
In the previous papers mentioned above, however, a safety constraint is defined using the (expected) cumulative value and the constraint satisfaction is \textit{not} guaranteed during the learning process~\cite{stooke2020responsive}.
Hence, most of the existing studies deal with less strict safety constraints than our study that requires the agent to satisfy safety constraints \textit{at every time step}.
There has been research aimed at guaranteeing safety at every time step, even during the learning process.
For example, \citet{turchetta2016safe} proposed notable algorithms that satisfy the safety constraint with high probability by inferring the safety function via a Gaussian process (GP).
Also, \citet{wachi2021safe} proposed its extended algorithm that models the safety function via a GLM, which can also deal with safe RL problems with binary feedback.
Though they succeeded in guaranteeing safety with high probability, their theoretical results are based on the assumptions of the known and deterministic state transition function.
It is essentially difficult to extend these algorithms to our problem settings with unknown and stochastic state transitions.
As existing work on safe RL with unknown stochastic transition, \citet{amani2021safe} proposed an algorithm for linear MDPs with safety constraints while \citet{bennett2023provable} proposed an algorithm for safe RL problems with binary safety feedback and stochastic transitions.
Although such work proposed excellent algorithms for challenging problems, the existence of a known safe policy is assumed for any state, which does not hold in many real-world applications as discussed in  Section~\ref{sec:introduction} (i.e., high-speed vehicle example).
Table~\ref{tab:problem} summarizes the problem settings considered in existing work and this paper.

\paragraph{Long-term safety.}

In the control community, long-term safety has been well-studied under the name of control barrier function (CBF, \citet{ames2019control}).
For any state $s$, a CBF is a continuously differentiable function $h(s)$ that defines a safe set $\{s: h(s) \ge 0 \}$, i.e., an invariant set where any trajectory starting inside the set remains within the set.
The CBF is to maintain safety during the learning process, which is particularly useful for keeping a manipulator within a given safe space or ensuring that a robot avoids obstacles.
This advantage is beneficial for RL settings, and \citet{cheng2019end} proposed a safe RL algorithm to guarantee long-term safety via CBFs.
Unfortunately, however, humans need to manually define proper CBFs and it is often hard to find them.
In addition, \citet{koller2018learning} proposed a learning-based model predictive control scheme that provides high-probability safety guarantees during the learning process under the assumption that both a dominant term of the state transition function and safe region are known a priori.

\section{Problem Statement}
\label{sec:problem}

We consider episodic finite-horizon CMDPs, which can be formally defined as a tuple
\begin{equation}
    \cM \coloneqq (\cS, \cA, P, T, r, g, s_1),
\end{equation}
where $\cS$ is a state space $\{s\}$,
$\cA$ is an action space $\{a\}$,
$P: \cS \times \cA \rightarrow \Delta(\cS)$ is an unknown, stochastic state transition function to map a state-action pair to a probability distribution over the next states,
$T \in \mathbb{Z}_+$ is a fixed length of each episode,
$r: \cS \times \cA \rightarrow [0, 1]$ is a (bounded) reward function,
$g: \cS \times \cA \rightarrow \{0, 1\}$ is an unknown binary safety function, and $s_1 \in \cS$ is the initial state.\footnote{We assume that reward function is known and deterministic, but all results presented here extend to unknown stochastic cases.}
Crucially, in this paper, the safety feedback is provided as a \textit{binary} value; that is, $g(s,a) = 1$ means that a state-action pair $(s,a)$ is safe, and otherwise $(s,a)$ is unsafe.
At the time step $t$ and the current state $s_t$, the agent takes the next action $a_t$, receiving the next state $s_{t+1} \sim P(\cdot \mid s_t, a_t)$ as well as the safety observation $g(s_t, a_t)$, until the terminal time step $T$.
We suppose that safety observations contain some independent zero-mean noise $n_t$.
We assume that the noise $n_t$ is sub-Gaussian with fixed (positive) parameters $\sigma \in \R_+$; that is, for all $t$, we have $\mathbb{E} \left [\, e^{\omega n_t} \mid \mathcal{G}_{t-1}\,\right] \le e^{\omega^2 \sigma^2/2}$,
where $\{\mathcal{G}_{t}\}$ is increasing sequences of sigma fields such that $n_t$ is $\mathcal{G}_{t}$-measurable with $\mathbb{E} \left[\, n_t \mid \mathcal{G}_{t-1}\, \right]=0$.
This assumption has been commonly made in previous work (e.g., \citet{NIPS2011_e1d5be1c}, \citet{li2017provably}).

A deterministic policy of an agent $\pi: \cS \rightarrow \cA$ represents a function to return actions. 
A metric of the quality of the policy $\pi$ is the following value function, i.e., the expected value of cumulative rewards, which is defined as
\begin{equation*}
 \sind{V}{\pi}{t}(s) \coloneqq \E_{\pi} \left[\, \sum_{\tau=t}^T r(s_\tau, a_\tau)  \, \bigg | \, s_t = s \,\right],
\end{equation*}
for all $s \in \cS$ and $t \in [T]$,
where the expectation $\E_\pi[\cdot]$ is taken over the trajectories $\{(s_\tau, a_\tau)\}_{\tau=t}^T$ induced by the policy $\pi$ and true state transition dynamics $P$.
We additionally define the following action-value function (i.e., Q-function) which means the expected value of total rewards when the agent starts
from state-action pair $(s,a)$ at step $t$ and follows policy $\pi$, which is represented as 
\begin{equation*}
    \sind{Q}{\pi}{t}(s,a) \coloneqq \E_{\pi} \left[\, \sum_{\tau=t}^T r(s_\tau, a_\tau) \, \bigg |\, s_t = s, a_t = a \,\right],
\end{equation*}
for all $(s,a) \in \cS\times \cA$ and $t \in [T]$.

A crucial point of this paper is that we wish the agent to take only \textit{safe} actions at every time step $t$; that is, the agent needs to take a safe action $a_t$ at a state $s_t$ that satisfies the safety constraint; that is, $g(s_t, a_t) = 1$.
As discussed in Section~\ref{sec:introduction}, however, at time $t$, the agent is required to execute safe actions in the long run so that there also will be future safe actions from time $t+1$ to $T$.
Hence, at every time step $t$, we impose the following safety requirement:
\begin{equation}
    \label{eq:constraint}
    \Pr \Bigl\{ g(s_\tau, a_\tau) = 1 \quad \forall \tau \in [t, T] \Big\} \ge 1 - \delta,
\end{equation}
where $\delta \in [0, 1]$ is a small positive scalar.
Our safety constraint is probabilistic since it is extremely difficult to guarantee safety almost surely (i.e., probability of $1$) due to the unknown stochastic state transition and safety functions.

\paragraph{Goal.}
Let us clearly describe the goal we wish to achieve in this paper.
The objective of the agent is to obtain the optimal policy $\pi^\star: \cS \rightarrow \cA$ to maximize the value function $V^\pi_t(s_t)$ under the safety constraint \eqref{eq:constraint} such that
\begin{align*}
    \label{eq:opt}
    \max_{\pi} V^\pi_t(s_t)
    \ \  \text{s.t.} \ \
    \Pr \Bigl\{ g(s_\tau, a_\tau) = 1 \ \ \forall \tau \in [t, T] \Big\} \ge 1 - \delta.
\end{align*}
It is quite hard to guarantee the satisfaction of the aforementioned constraint.
It is because even if the agent executes an action $a_t$ at time $t$ and state $s_t$ such that
\begin{equation}
    \label{eq:short_constraint}
    \Pr\Bigl\{ g(s_t, a_t) = 1 \Big\} \ge 1 - \delta,
\end{equation}
there may \textit{not} be any viable action at $s_{t+1} \sim P(s_t, a_t)$ and further future states.
Thus, the agent must execute an action $a_t$ to guarantee the constraint satisfaction not only for $(s_t, a_t)$ but also for $(s_\tau, a_\tau)$ for all $\tau \in [t+1, T]$.
Our safety constraint \eqref{eq:constraint} is challenging, which we will call the \textit{long-term} safety constraint in the rest of this paper, while we will call~\eqref{eq:short_constraint} the \textit{instantaneous} safety constraint.

\paragraph{Difficulties and assumptions.}

The aforementioned problem we wish to solve has several difficulties.
First, if the binary safety function does not exhibit any regularity, it is impossible to infer the safety of state-action pairs.
For example, if the safety function value is totally random, we can neither foresee danger nor guarantee safety.
In addition, we suppose the state transition is stochastic and unknown a priori despite that the agent must guarantee the satisfaction of the long-term safety constraint.
This difficulty requires us to explicitly incorporate the stochasticity of the state transition and its influence on future safety.

For the first difficulty, we assume that the safety function can be modeled as a GLM to deal with binary safety feedback.
GLMs have been studied for sequential decision-making problems with binary feedback especially in (stateless) multi-armed bandit literature~\citep{filippi2010parametric,li2017provably,faury2020improved} under the name of logistic bandits.
Also, in (stateful) RL settings, \citet{wachi2021safe} addressed a safe RL problem where the safety function is subject to a GLM under the assumption that the state transition is a priori known and deterministic.
We now make the following assumption of the GLM structure of the safety function.
\begin{assumption}
\label{assumption:linear}
There exists a known feature mapping function $\bphi: \cS \times \cA \rightarrow \R^m$, unknown coefficient vectors $\bm{w}^\star \in \R^m$, and a fixed, strictly increasing (inverse) link function $\mu: \R \rightarrow [0, 1]$ such that
\begin{align}
\label{eq:gl_safety}  
    \E[\, g(s, a) \mid s, a \,] = \mu\bigl(f^\star(s,a)\bigr),
\end{align}
for all $(s, a) \in \cS \times \cA$,
where $f^\star: \cS \times \cA \rightarrow \R$ is a linear predictor defined as
\begin{align}
    f^\star(s, a) \coloneqq \iprod{\bphi(s,a)}{\bm{w}^\star}, \quad \forall (s, a) \in \cS \times \cA. 
\end{align}
Without loss of generality, we further assume $\norm{\bphi(s, a)} \le 1$ for all $(s,a) \in \cS \times \cA$ and $\norm{\bm{w}^\star} \le \sqrt{m}$.
\end{assumption}
\noindent
In the case of the binary safety function, a suitable choice of the link function is $\mu(x) = \exp(x)/(1 + \exp(x))$, leading to the logistic regression model.
GLMs are more general models, and one can verify that linear and integer-valued functions are special cases of GLMs with $\mu(x) = x$ and $\mu(x) = \exp(x)$ leading to the linear regression model and the Poisson regression model, respectively; hence, our method can be extended to other problem settings than the binary safety function.

In addition to the boundedness assumption on the feature vectors and safety function values, we make the following assumption regarding the link function.
\begin{assumption}
\label{assumption:link}
The link function $\mu$ is twice differentiable, and the first and second-order derivatives are respectively bounded.
Also, the link function $\mu$ satisfies $\xi = \inf_{\| \bm{w} - \bm{w}^\star \| \le 1, \|\bm{\phi}\| \le 1} \dot{\mu}(\iprod{\bm{\phi}} {\bm{w}}) > 0$.
\end{assumption}

By making Assumptions~\ref{assumption:linear} and \ref{assumption:link}, it is possible to guarantee the satisfaction of the instantaneous safety constraint~\eqref{eq:short_constraint} at the current time step $t$ if there are feasible actions, as conducted in \citet{wachi2021safe}.
In this paper, however, the agent must guarantee safety until the terminal time step $T$ (i.e., long-term safety constraint) under stochastic state transition, which requires us to make further assumptions.
If the feature mapping function drastically changes with minor differences in state-action pairs, it is extremely difficult to continue to guarantee safety until $T$.
Thus, we assume the regularity of the feature mapping function as a form of Lipschitz continuity, which is written as follows:
\begin{assumption}
\label{assumption:lipschitz_feature}
For all $s, \bar{s} \in \cS$ and $a, \bar{a} \in \cA$, the feature mapping function $\bphi(\cdot, \cdot)$ is Lipschitz continuous with a constant $L_\phi \in \R_{+}$; that is,
\begin{equation}
    \norm{\bphi(s, a) - \bphi(\bar{s}, \bar{a})} \le L_\phi \cdot d_{\cS\cA}((s, a), (\bar{s}, \bar{a})),
\end{equation}
where $d_{\cS\cA}(\cdot, \cdot)$ is a distance metric on $\cS \times\cA$.
For ease of exposition, we assume that $d_{\cS\cA}$ satisfies $d_{\cS\cA}((s, a), (\bar{s}, \bar{a})) = d_\cS(s, \bar{s}) + d_\cA(a, \bar{a})$.
\end{assumption}
\noindent
Intuitively, this assumption implies that, for similar state-action pairs $(s,a)$ and $(\bar{s}, \bar{a})$, the features $\bphi(s, a)$ and $\bphi(\bar{s}, \bar{a})$ also exhibit similar values.
This assumption is related to the common assumption in RL literature as represented by Lipschitz MDP~\cite{asadi2018lipschitz,ok2018exploration}.

Similarly, at a current state $s$, if the next state $s' \sim P(\cdot \mid s, a)$ induced by an ``insignificant'' action $a$ (that tries to maintain the status quo) is far from $s$, the safety may drastically changes.
Hence, we assume the existence of a Lipschitz-continuous conservative policy $\pi^\sharp: \cS \rightarrow \cA$ to suppress the state transition distance within a certain value.
We then assume that, as far as similar policies to the conservative policy are executed, the state-transition distance can be suppressed.
Specifically, for any policy $\pi$, we assume that the (one-step) state transition from time $t$ to $t+1$ is upper-bounded according to the divergence between the actions taken by $\pi$ and $\pi^\sharp$.
\begin{assumption}
    \label{assumption:stabilizing}
    Let $L_\sharp \in \R_+$ be a positive scalar.
    There exists a known $L_\sharp$-Lipschitz continuous policy $\pi^\sharp: \cS \rightarrow \cA$ such that, for any states $s, \bar{s} \in \cS$,
    \begin{equation}
        d_\cA(\pi^\sharp(s) - \pi^\sharp(\bar{s})) \le L_\sharp \cdot d_\cS(s, \bar{s}).
    \end{equation}
    Also, with a positive scalar $\eta \in \R_{+}$, for any policy $\pi: \cS \rightarrow \cA$, the following inequality holds for all $s \in \cS$:
    \begin{equation*}
        \max_{s' \sim P(\cdot \mid s, \pi(s))} d_\cS(s, s') \le \bar{d} + \eta \cdot d_\cA(\pi(s), \pi^\sharp(s)).
    \end{equation*}
\end{assumption}

\begin{remark}
    Assumption~\ref{assumption:stabilizing} implies that the conservative policy $\pi^\sharp$ keeps the amount of the (one-step) state transition within a certain distance $\bar{d} \in \R_{+}$; that is,
    \begin{equation*}
        \max_{s' \sim P(\cdot \mid s, \pi^\sharp(s))} d_\cS(s, s') \le \bar{d}.
    \end{equation*}
\end{remark}

\noindent
In the case of stochastic policies, we can use Kantorovich distance $K(\cdot, \cdot)$ to define the Lipschitz continuity of a policy; that is, $K(\pi(\cdot \mid s), \pi(\cdot \mid \bar{s})) \le L_\pi \cdot d_\cS(s, \bar{s})$.
Thus, the following theoretical analysis can be extended to stochastic policy settings.
This assumption is valid in many physical systems (e.g., control-affine systems).
Intuitively, when the policy $\pi$ is similar to the conservative policy $\pi^\sharp$, the upper bound of the state transition is guaranteed to be small.
Assumption~\ref{assumption:stabilizing} implies that when $\pi = \pi^\sharp$, the state transition distance is always less than or equal to $\bar{d}$.
Also, as $\pi$ becomes far from $\pi^\sharp$, the distance can be larger with respect to the term $\eta \cdot d_\cA (\pi, \pi^\sharp$). 
The existence of $\pi^\sharp$ is not restrictive in practice for a number of applications, and similar notions have been adopted in many existing studies under the name of the stable or telescoping policies \cite{lin2021perturbation,tsukamoto2021contraction}.
For instance, with the autonomous vehicle, one may select 
$\pi^\sharp$ as the one to move it at a low constant speed, 
and $\pi$ is optimized such that it can move faster under the safety constraint.

\section{Characterizing Safety}
\label{sec:preliminary}

Based on the problem settings and assumptions presented in Section~\ref{sec:problem}, we now present how to guarantee long-term safety.
Optimism and pessimism are essential notions in RL.
Conventionally, being optimistic has been well-adopted in online RL literature under the name of optimism in the face of uncertainty principle~\cite{strehl2008analysis,auer2007logarithmic}.
In contrast, pessimism is also significant when an RL agent is trained from offline data~\cite{jin2021pessimism,buckman2020importance} or needs to satisfy safety constraints~\cite{bura2022dope}.
A natural way to incorporate optimism and pessimism is to derive the upper and lower bounds of the functions of interest, which can be conducted in a way backed by theory.

This paper expresses the upper and lower bounds in two ways.
The first is inferred by the GLMs.
While the advantage of this approach is to provide accurate estimation once a larger amount of dataset has been collected, the uncertainty term tends to be loose in the early phase of the training.
The second is based on Lipschitz continuity.
In contrast to the GLM-based approach, this approach provides moderate bounds regardless of the amount of collected data, which is typically useful in the early phase of training.
Thus, intuitively, we aim to continue to derive tight bounds by deriving them using the approach based on Lipschitz continuity in the early phase and that based on the GLMs in the later phase.

\subsection{Confidence Intervals Inferred by GLMs}
\label{sec:bound_glm}

We first present how to obtain theoretically-guaranteed confidence bounds inferred by the GLMs.
Hereinafter, let the design matrix be
$W_n = \sum_{j=1}^n \bm{\phi}(s_j, a_j) \bm{\phi}(s_j, a_j)^\top$, where $n \in \mathbb{Z}_+$ is the total number of data.
Also, the weighted $L_2$-norm of $\bm{\phi}$ associated with $W_n^{-1}$ is given by
$\onlynorm{\bm{\phi}}_{W_n^{-1}} \coloneqq \sqrt{\bm{\phi}^\top W_n^{-1} \bm{\phi}}$.
Here, the maximum-likelihood estimators (MLE) denoted as $\hat{\bm{w}}$ is calculated by solving the following equation:
\begin{equation*}
    \sum_{j=1}^n \bigl(g(s_j, a_j) - \mu(\la \bphi(s_j, a_j), \bm{w} \ra) \bigr) \bphi(s_j, a_j) = 0,
\end{equation*}
Based on \citet{li2017provably}, the following lemma regarding the confidence bounds on $f^\star$ holds.
\begin{lemma}
\label{lemma:confidence_bound}
Let $\Delta > 0$ be given and $\beta = \frac{3 \sigma}{\xi} \sqrt{\log \frac{3}{\Delta}}$.
Then, with a probability of at least $1-\Delta$, the MLE satisfies
\begin{alignat*}{2}
\abs{f^\star(s,a) - \iprod{\bphi(s,a)}{\hat{\bm{w}}}} \le \beta \cdot \onlynorm{\bphi(s,a)}_{W_n^{-1}},
\end{alignat*}
for all $(s,a) \in \cS \times \cA$.
\end{lemma}
\noindent
Therefore, at time $t$ and state $s_t$, by choosing the next action $a_t$ such that $\iprod{\bphi(s_t, a_t)}{\hat{\bm{w}}} - \beta \cdot \onlynorm{\bphi(s_t, a_t)}_{W_n^{-1}} \ge z$, we can also guarantee the satisfaction of $f^\star(s_t, a_t) \ge z$ with high probability, where $z \in \R$ is a certain threshold.

\subsection{Bounds by Lipschitz Continuity}
\label{sec:bound_lipschiz}

We then present the upper and lower bounds inferred by the Lipschitz continuity.
Let us first define an important variable $x_t \in \R_+$ called maximum divergence from the conservative policy (MDCP) such that 
\begin{equation}
    \label{eq:policy_x_t}
    d_\cA(\pi(s_t), \pi^\sharp(s_t)) \le x_t, \quad \forall t \in [T].
\end{equation}
The MDCP indicates how far the action taken by $\pi$ is from that by $\pi^\sharp$.
Hereinafter, the summation of this new variable $x_t$ plays a critical role when dealing with the long-term safety constraint, and thus we define $X_{t_1}^{t_2} \coloneqq \sum_{\tau=t_1}^{t_2} x_\tau$ for any time steps $t_1, t_2 \in [T]$ with $t_1 < t_2$.
We have the following two lemmas regarding the (true) safety linear predictor $f^\star$.
See Appendices \ref{proof:A_2} and \ref{proof:A_3} for the proofs.

\begin{lemma}
    \label{lemma:f_t_T}
    Suppose the policy $\pi$ satisfies \eqref{eq:policy_x_t}.
    Let $L_1$, $L_2$, and $L_3$ be constants that are respectively defined as 
    \[
        L_1 \coloneqq \sqrt{m} \cdot L_\phi, L_2 \coloneqq \bar{L}_\sharp \cdot \bar{d}, L_3 \coloneqq 2 + \eta \bar{L}_\sharp.
    \]
    Set $\bar{t} \coloneqq T - t$ and
    recall that $X_{t+1}^{T-1} \coloneqq \sum_{\tau=t+1}^{T-1} x_\tau$.
    Finally, with $x_{t:T} \coloneqq x_t, x_{t+1}, \ldots, x_T$, define
    \begin{equation*}
        \mathcal{F}(t, x_{t:T}) \coloneqq L_1 \left\{L_2 \bar{t} + (L_3-1) x_t + L_3 X_{t+1}^{T-1} + x_{T} \right\}.
    \end{equation*}
    Then, we have
    \begin{align*}
        \abs{f^\star(s_T, \pi(s_T)) - f^\star(s_t, \pi(s_t))}
        \le \mathcal{F}(t, x_{t:T}).
    \end{align*}
\end{lemma}
\noindent
Intuitively, Lemma~\ref{lemma:f_t_T} characterizes the present-to-future difference in terms of $f^\star$, which provides us the lower bound of the future safety linear predictor.
\begin{lemma}
    \label{lemma:f_1_t}
    Define $f^\sharp(s) \coloneqq f^\star(s, \pi^\sharp(s))$ for all $s \in \cS$.
    Also, suppose the policy $\pi$ satisfies \eqref{eq:policy_x_t}.
    Then, we have
    \begin{align*}
        \abs{f^\star(s_t, \pi(s_t)) - f^\sharp(s_1)}
        \le L_1 \left\{L_2 \, t + L_3 X_{2}^{t-1} + x_{t} \right\}.
    \end{align*}
\end{lemma}
\noindent
In contrast to Lemma~\ref{lemma:f_t_T}, Lemma~\ref{lemma:f_1_t} characterizes the past-to-present difference in terms of $f^\star$.
Using this lemma, we infer the lower bound of $f^\star$ at the current time step $t$.

\begin{figure*}[t]
    \centering
    \begin{subfigure}[b]{0.33\textwidth}
        \centering
        \includegraphics[width=\textwidth]{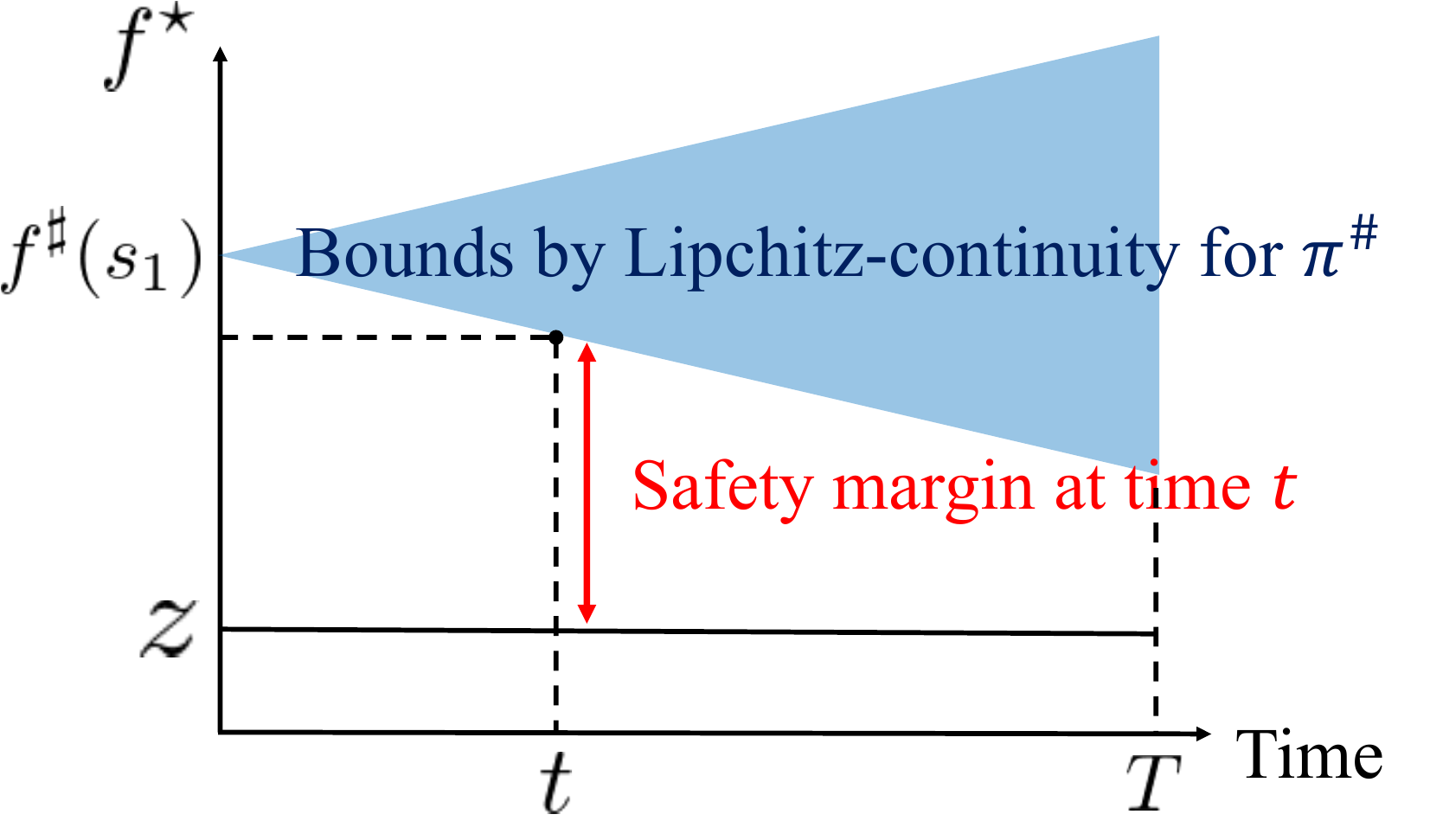}
        \caption{Conservative policy}
        \label{fig:point_return}
    \end{subfigure}
    \hfill
    \begin{subfigure}[b]{0.33\textwidth}
        \centering
        \includegraphics[width=\textwidth]{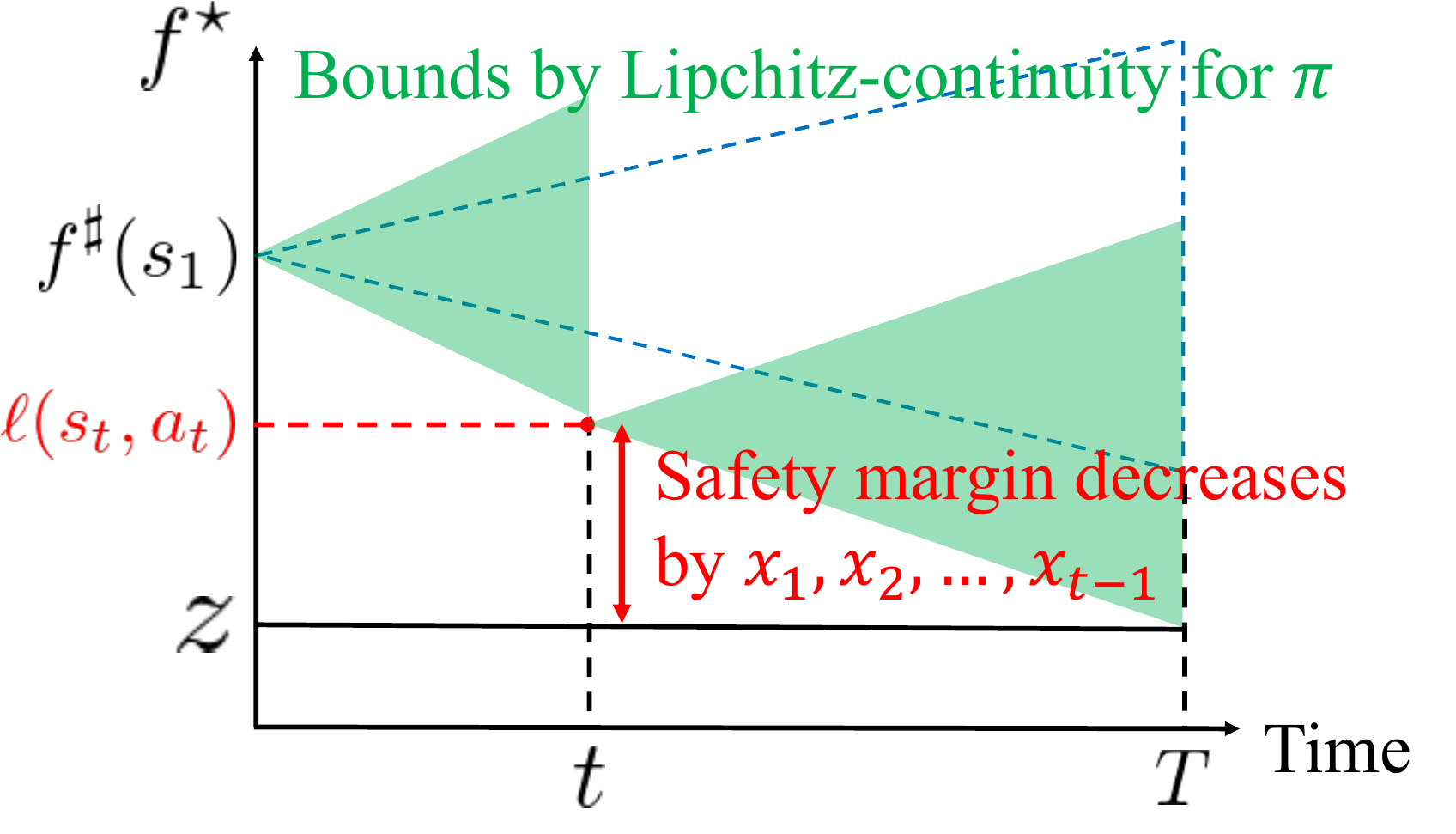}
        \caption{Early phase (Lipschitz)}
        \label{fig:point_avecost}
    \end{subfigure}
    \hfill
    \begin{subfigure}[b]{0.33\textwidth}
        \centering
        \includegraphics[width=\textwidth]{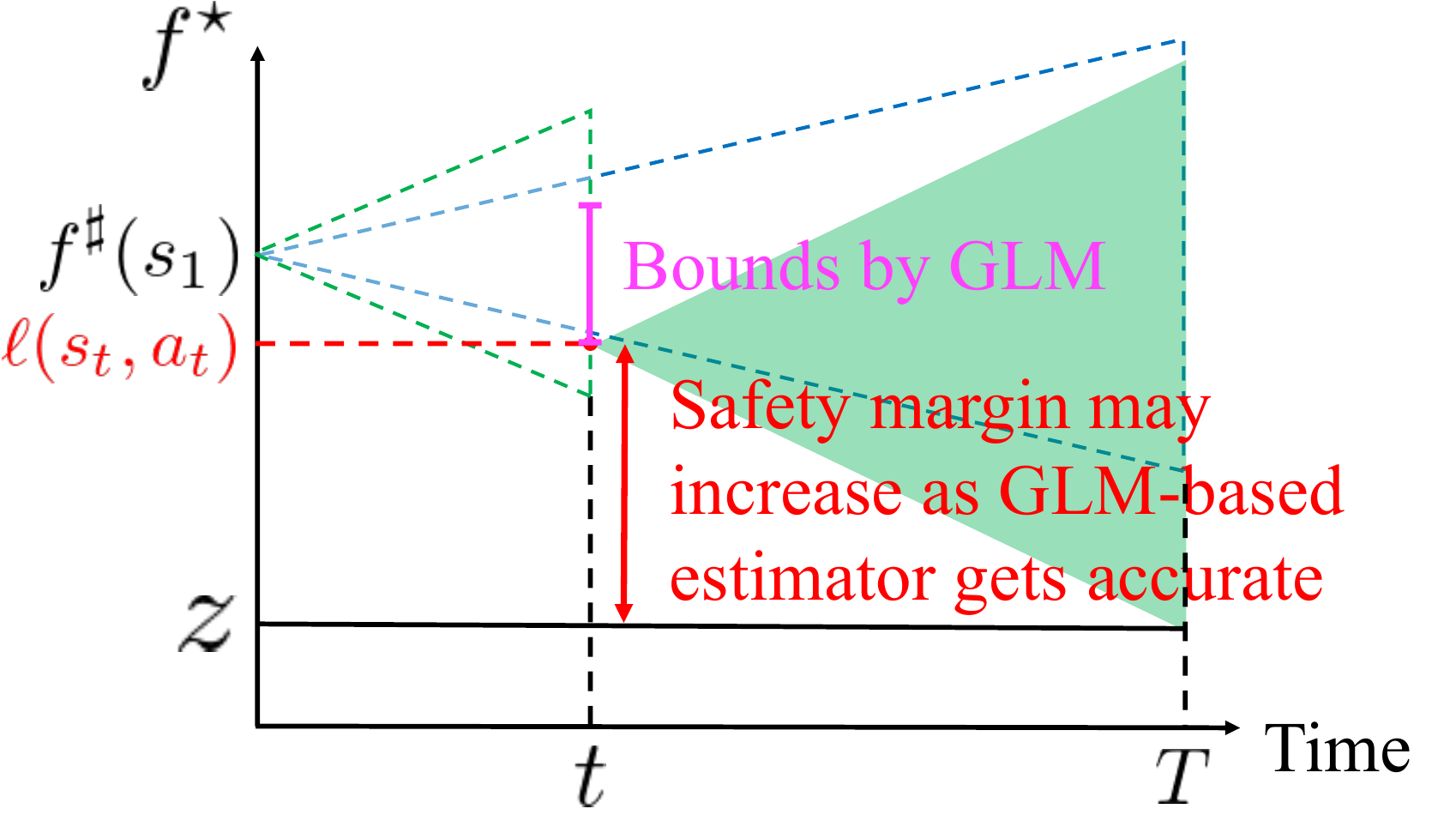}
        \caption{Later phase (GLM)}
        \label{fig:point_maxcost}
    \end{subfigure}
    \caption{(a) Bounds by Lipschitz continuity for the conservative policy. (b) In the early phase of training, the lower bound of the safety linear predictor at time $t$ is typically characterized by the Lipschitz continuity, which decreases depending on the $x_1, x_2, \ldots, x_{t-1}$. Depending on the safety margin at time $t$, we need to control $x_{t}, x_{t+1}, \ldots, x_T$ for ensuring future safety. (c) As the training proceeds, the lower bound of the safety linear predictor can potentially be characterized by the GLMs, and the safety margin may increase.}
    \label{fig:long_term_safety}
\end{figure*}

\subsection{Resulting Lower Bound of $\bm{f^\star}$}

As we discussed previously, it is a simple yet powerful way to use the safety lower bound for introducing pessimism in safe RL.
Let $\ell: \cS \times \cA \rightarrow \R$ denote a lower bound of the true safety linear predictor, $f^\star$. 
To obtain a tighter bound, this paper combines the two lower bounds presented in Section~\ref{sec:bound_glm} and \ref{sec:bound_lipschiz}, respectively.
Specifically, based on Lemma~\ref{lemma:confidence_bound} and \ref{lemma:f_1_t}, we obtain the following tighter bound:
\begin{equation}
    \ell(s_t, a_t) \coloneqq \max \bigl(\ell_\text{GLM}(s_t, a_t), \ell_\text{Lipschitz}(s_t, a_t) \bigr),
\end{equation}
where $\ell_\text{GLM}: \cS \times \cA \rightarrow \R$ and $\ell_\text{Lipschitz}: \cS \times \cA \rightarrow \R$ are pessimistic safety linear predictors inferred by GLM and Lipshitz continuity, which are respectively defined as
\begin{align*}
    \ell_\text{GLM}(s_t, a_t) 
    \coloneqq &\ \iprod{\bphi(s_t, a_t)}{\hat{\bm{w}}} - \beta \cdot \onlynorm{\bphi(s_t, a_t)}_{W_n^{-1}}, \\
    \ell_\text{Lipschitz}(s_t, a_t)
    \coloneqq &\ f^\sharp(s_1) - L_1 \left\{L_2\, t + L_3 X_{1}^{t-1} + x_{t} \right\}.
\end{align*}

\subsection{Long-term Safety Guarantee}

We present theoretical results regarding the lower bound of the safety linear predictor $f^\star$, which leads to the long-term safety guarantee defined by \eqref{eq:constraint}.
We now present a lemma regarding the safety linear predictor at time $t$.
\begin{lemma}
    \label{lemma:lower_bound}
    At every time step $t \in [T]$, we have
    \begin{equation}
        \label{eq:f_t}
        f^\star(s_t, a_t) \ge \ell(s_t, a_t)
    \end{equation}
    with a probability of at least $1-\Delta$.
\end{lemma}
\noindent
This lemma implies that we can guarantee the instantaneous safety constraint \eqref{eq:short_constraint} by choosing the next action such that 
\begin{equation}
    \ell(s_t, a_t) \ge z, \quad \text{with} \quad \mu(z) = 1 - \delta,
\end{equation}
with a probability of at least $1 - \Delta$.

In this paper, however, we need to additionally require the satisfaction of the long-term safety constraint; thus, we are particularly interested in future safety.
We now provide the following lemma in terms of the pessimistic safety linear predictor at the terminal time step $T$:
\begin{lemma}
    \label{lemma:f_ell_T}
    Recall $T$ is the terminal time step and set $\bar{t} = T - t$.
    At every time step $t \in [T]$, we have
    \begin{align*}
        \label{eq:f_t_T}
        f^\star(s_T, a_T)
        \ge &\ \ell(s_t, a_t) - \mathcal{F}(t, x_{t:T}),
    \end{align*}
    with a probability of at least $1-\Delta$.
\end{lemma}

\begin{corollary}
    \label{corollary:safety}
    Suppose, at state $s_t$, the agent with a policy $\pi$ executes the action $a_t$ while tuning $x_t, x_{t+1}, \ldots, x_T$ so that
    \begin{equation}
        \label{eq:condition_safety}
        \ell(s_t, a_t) - \mathcal{F}(t, x_{t:T}) \ge z
    \end{equation}
    holds.
    Then, for all $\tau \in [t, T]$, there exist safe state-action pairs $(s_\tau, a_\tau)$ such that: 
    \begin{equation}
        f^\star(s_\tau, a_\tau) \ge z, \quad \tau \in[t, T], 
    \end{equation}
    with a probability of at least $1 - \Delta$.
\end{corollary}
\noindent
The proofs of Lemma~\ref{lemma:f_ell_T} and Corollary~\ref{corollary:safety} are written in 
Appendix \ref{appendix:A_5}.

Finally, we present a main theorem on the long-term safety constraint.
Specifically, we guarantee that an agent continues to take safe actions from time $t$ to $T$ with a higher probability than a predefined threshold, by properly tuning the MDCPs, $x_\tau$ for all $\tau \in [t, T]$.
\begin{theorem}
    \label{theorem:safety}
    Suppose, at state $s_t$, the agent executes the action $a_t$ while tuning the MDCPs $x_t, x_{t+1}, \ldots, x_T$ so that \eqref{eq:condition_safety} holds.
    Set $\delta \coloneqq 1 - (1 - \mu(z))^{\bar{t}}$.
    Then, we have
    \begin{align*}
    \label{eq:opt}
        \Pr \Bigl\{ g(s_\tau, a_\tau) = 1 \ \ \forall \tau \in [t, T] \Big\} \ge 1 - \delta, \quad \forall t \in [T],
    \end{align*}
    --- i.e. the long-term safety constraint is satisfied --- with a probability of at least $1-\Delta$.
\end{theorem}
\noindent
This theorem guarantees that at every time step $t$, the agent can take safe actions from $t$ to $T$ with high probability, despite unknown, stochastic state transition and binary safety feedback.
The proof sketch is as follows.
By Corollary~\ref{corollary:safety}, when \eqref{eq:condition_safety} is satisfied, $f^\star(x_\tau, a_\tau) \ge z$ holds for all $\tau \in [t, T]$ with high probability; that is, the existence of future safe actions are guaranteed with high probability.
Theorem~\ref{theorem:safety} provides a stricter safety guarantee than the one in existing safe RL literature with instantaneous safety constraints such as \citet{wachi2021safe}.
If we tried to guarantee safety while using the instantaneous constraint~\eqref{eq:short_constraint}, the agent would fall into worse situations and then lose the choices of safe actions due to the stochastic state transition.

\begin{figure*}[t]
    \centering
    \begin{subfigure}[b]{0.33\textwidth}
        \centering
        \includegraphics[width=\textwidth]{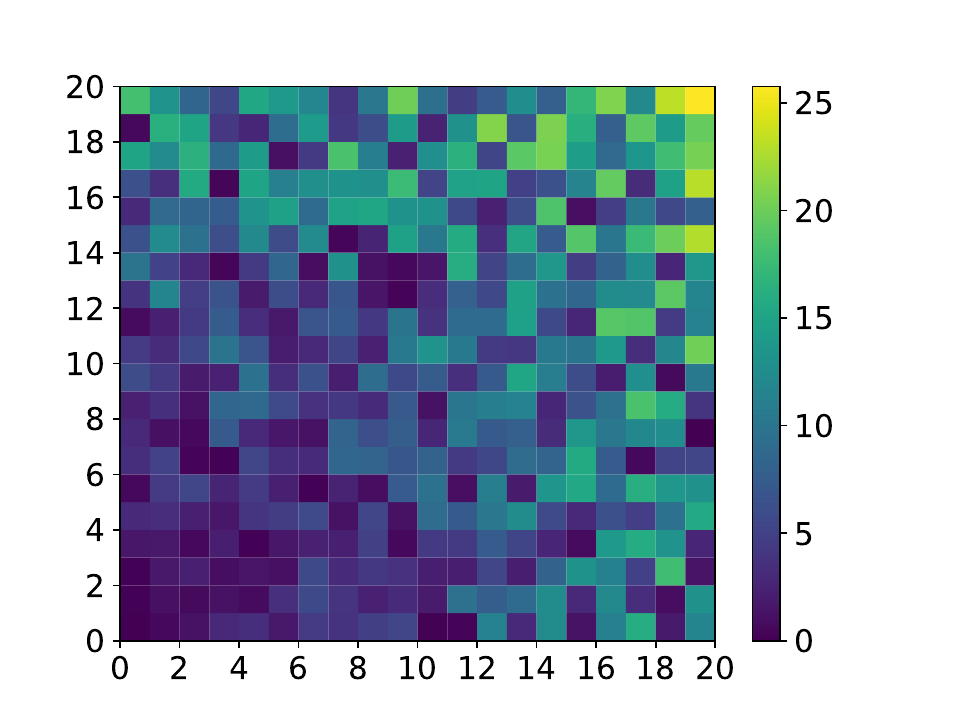}
        \caption{Reward function.}
        \label{fig:reward}
    \end{subfigure}
    \hfill
    \begin{subfigure}[b]{0.33\textwidth}
        \centering
        \includegraphics[width=\textwidth]{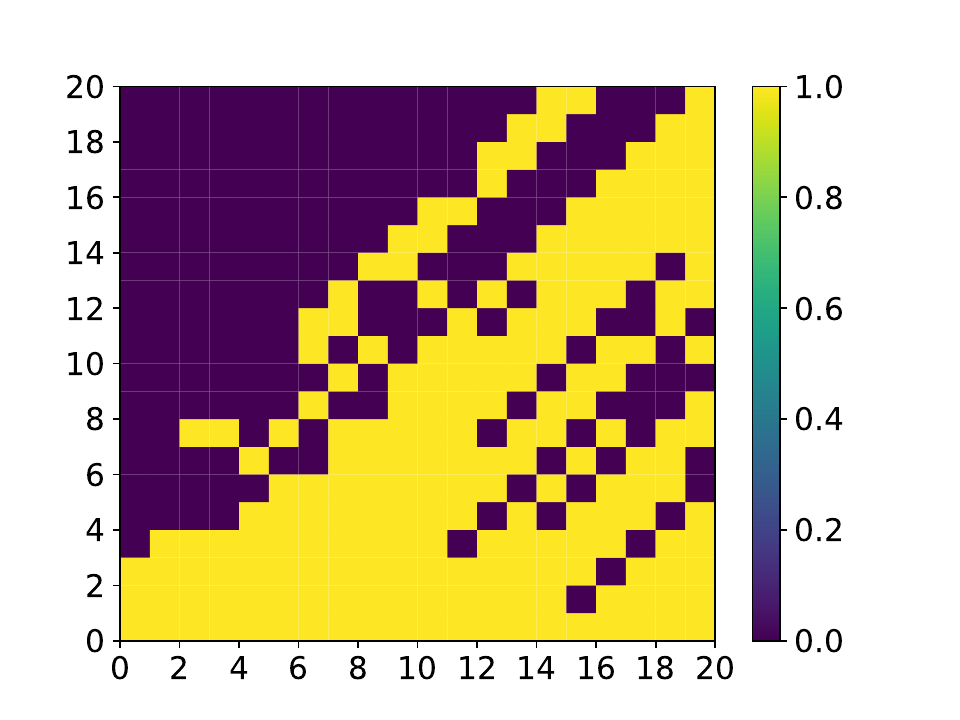}
        \caption{Binary safety function.}
        \label{fig:safety}
    \end{subfigure}
    \hfill
    \begin{subfigure}[b]{0.33\textwidth}
        \centering
        \includegraphics[width=\textwidth]{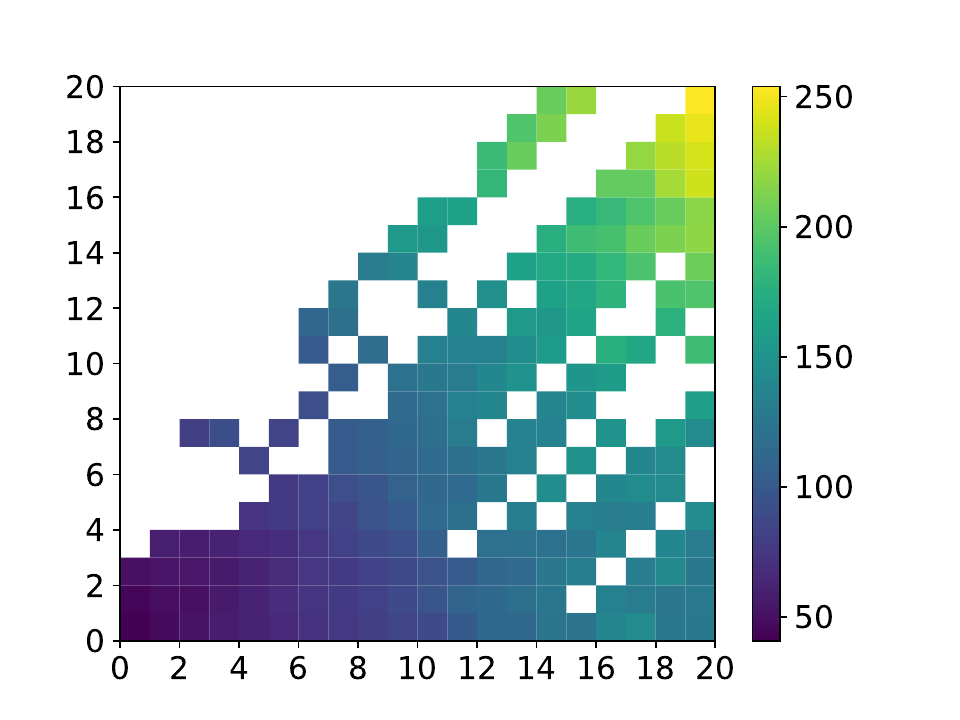}
        \caption{Value function with safety.}
        \label{fig:value}
    \end{subfigure}
    \caption{Example reward, binary safety, and value functions. In this paper, we consider a safe RL problem with binary safety feedback; thus, there is an unsafe region (the white region in the (c)) where the agent is not allowed to visit.}
    \label{fig:experiment}
\end{figure*}

\section{LoBiSaRL Algorithm}
\label{sec:method}

We finally propose our \algo~algorithm.
The algorithm flow is shown in Algorithm~\ref{alg:algorithm}.
Based on Theorem~\ref{theorem:safety}, we should solve the following policy optimization problem under a (conservative) long-term safety constraint:
\begin{align*}
    \max_\pi V_t^\pi(s_t) \quad \text{subject to} \quad \ell(s_t, a_t) - \mathcal{F}(t, x_{t:T}) \ge z.
\end{align*}
Note that, the term $\mathcal{F}(t, x_{t:T})$ can be transformed into
\begin{align*}
\mathcal{F}(t, x_{t:T}) = L_1 \cdot \biggl\{
\underbrace{L_2 \bar{t}}_{\text{(A)}}
+ \underbrace{(L_3-1) x_t}_{\text{(B)}}
+ \underbrace{L_3 X_{t+1}^{T-1} + x_{T}}_{\text{(C)}} \biggr\}.
\end{align*}
The above inequality can be interpreted as follows.
(A) is an inevitable term that even the conservative policy $\pi^\sharp$ cannot avoid.
(B) depends only on the current action at time $t$, and (C) depends on the future actions from time $t+1$ to $T$.
We can make (B) and (C) terms zero by executing the same actions as the conservative policy.

\begin{algorithm}[t]
    \caption{Long-term Binary Safe RL (\algo)}
    \label{alg:algorithm}
    \begin{algorithmic}[1]
    \begin{small}
        \STATE \textbf{Input:} Initial Lagrange multiplier $\lambda_1$. Constants $L_1$, $L_2$, and $L_3$. conservative policy $\pi^\sharp$.
        \FOR{iteration $i = 1, 2, \ldots$}
            \FOR{time $t = 1, 2, \ldots, T$}
                \STATE $\pi_t \leftarrow \argmax_\pi V_t^\pi(s_t) - \lambda_i \left(-x_t + L_3 X_{t}^{T-1} + x_{T}\right)$
                \STATE $A_t \leftarrow \{a \in \!\cA \mid \ell(s_t, a) - L_1 \{L_2 \bar{t} + (L_3 -1) x_t \} \ge z \}$
                \IF{$\pi_t(s_t) \in A_t$}
                    \STATE $a_t \leftarrow \pi_t(s_t)$
                \ELSE
                    
                    \STATE $a_t \leftarrow \argmin_{a \in A_t} \norm{a - \pi_t(s_t)}$
                \ENDIF
                \STATE Take $a_t$ and then receive a next state $s_{t+1} \sim P(s_t, a_t)$, reward $r(s,a)$, and (binary) safety $g(s,a)$.
                \STATE Update value function $V_t^\pi$
            \ENDFOR
            \STATE $H_i \coloneqq \min_t \left(\ell(s_t, \pi(s_t)) - z \right)$
            \STATE Update the Lagrange multiplier to $\lambda_{i+1}$ based on $H_i$
        \ENDFOR
    \end{small}
    \end{algorithmic}
\end{algorithm}

A key to solving the aforementioned constrained policy optimization problem is how we tune $x_\tau$ for all $\tau = [t, T]$.
Intuitively, we want to set $x$ to be large in terms of reward maximization while $x$ should be small in terms of long-term safety guarantee.
Hence, we use a Lagrangian method to simultaneously maximize the expected cumulative reward while tuning the magnitude of $x$ for the satisfaction of the safety constraint.
Specifically, with a Lagrange multiplier $\lambda \in \R_+$, we solve the following max-min problem:
\begin{align}
    &\max_\pi \min_{\lambda \ge 0} \ V_t^\pi(s_t) - \lambda \cdot (-x_t + L_3 X_{t}^{T-1} + x_{T}).
\end{align}
By setting $\lambda$ large, we enforce the agent to make $x$ small and thus execute similar actions to the conservative policy.
When the conservative policy has much safety margin, the agent should explore the state and action spaces while taking more different actions.
The degree of freedom is optimized by means of the Lagrange multiplier $\lambda$.

For the current policy $\pi$, the minimum safety requirement that the agent needs to satisfy at every time step $t$ is
\begin{align}
    \label{eq:safety_constraint_checked}
    \ell(s_t, \pi(s_t)) - L_1 \left\{ L_2 \bar{t} + (L_3 - 1) x_{t} \right\} \ge z.
\end{align}
The aforementioned inequality is derived by setting the (C) term to be $0$, which corresponds to executing the same actions to the conservative policy from time $t+1$ to $T$.
In other words, at time step $t$, the next action $a_t$ must be chosen with the following ``safe'' action set:
\begin{align*}
    \label{eq:safe_action_set}
    A_t \coloneqq \{a \in \cA \mid \ell(s_t, a) - L_1 \{L_2\bar{t} + (L_3 -1) x_t \} \ge z \}.
\end{align*}
\noindent
To optimize the Lagrange multiplier $\lambda$, we define the following minimum safety margin at $i$-th episode:
\begin{equation*}
    H_i \coloneqq \min_t \left(\ell(s_t, \pi(s_t)) - z \right).
\end{equation*}
When $H_i$ is large, the agent is allowed to explore further by taking different actions from the conservative policy.
In contrast, when $H_i$ is small, the agent needs to prioritize safety without diverging from the conservative policy.

\section{Experiments}

In this section, we evaluate the performance of \algo~in a synthetic grid-world environment.

\paragraph{Settings.}
This environment is $20 \times 20$ square grids in which reward and safety functions are randomly generated.
To avoid trivial situations where the optimal policy wanders around the initial position $(0, 0)$, we generate the reward function so that the reward-rich region is far from the initial state.
The safety function is generated so as to follow a GLM, and the agent receives the binary safety feedback.
At every time step, the agent takes an action from four action candidates (up, right, down, left).
Also, the state transition function is stochastic; thus, the agent can go in the intended direction 80\% of the time (if there is no wall).
We provide $10$ initial samples for initializing the GLM and set $T=50$.

\paragraph{Baselines.}
We compare the performance of \algo~with four baselines.
The first baseline is called \textsc{Random} agent, which randomly chooses the next action without any consideration of reward and safety.
The second is a \textsc{Unsafe} agent.
This agent purely maximizes the cumulative reward while ignoring the safety issues.
The third baseline is a \textsc{Linear} agent.
This algorithm is based on \citet{amani2021safe} to model the safety function via a linear model.
The final baseline is a \textsc{Instantaneous} agent.
This algorithm only considers the instantaneous safety constraint~\eqref{eq:short_constraint} as in \citet{wachi2021safe} and cannot guarantee the satisfaction of the long-term constraint in an environment with the stochastic state transition.

\begin{table}[t]
\centering
\begin{small}
\begin{tabular}{lrr}
\toprule
& Reward & Unsafe actions \\
\midrule
\textsc{Random} & $0.32 \pm 0.24$ & $23.2 \pm 10.3$ \\
\textsc{Unsafe} & $\bm{1.00 \pm 0.00}$ & $26.8 \pm 13.6$ \\
\textsc{Linear} & $0.73 \pm 0.13$ & $18.3 \pm 5.7$  \\
\textsc{Instantaneous} & $0.86 \pm 0.10$ & $3.3 \pm 2.2$ \\
\algo~(Ours) & $0.76 \pm 0.12$ & $\bm{0.0 \pm 0.0}$ \\
\bottomrule
\end{tabular}
\end{small}
\caption{Experimental results. Reward is normalized with respect to \textsc{Unsafe} agent.}
\label{tab:result}
\end{table}

\paragraph{Results.}

Table~\ref{tab:result} summarizes our experimental results.
To obtain the results, we run each algorithm while generating $100$ different random environments.
As for safety, \algo~is the only algorithm to guarantee the satisfaction of the safety constraint in the long run.
\textsc{Random}, \textsc{Unsafe}, and \textsc{Linear} execute a lot of unsafe actions.
\textsc{Instantaneous} agent is much safer than the above three baselines but sometimes violates the safety constraint due to the stochasticity of the environment.
In contrast, \algo~is often too conservative and the performance in terms of reward is worse than \textsc{Instantaneous}.
Given that \algo~is an algorithm for safety-critical applications, however, it would be more important to guarantee long-term safety if the performance degradation is minor in terms of reward.

\section{Conclusion}

We formulate a safe RL problem with stochastic state transition and binary safety feedback and then propose an algorithm called \algo.
This algorithm maximizes the expected cumulative reward while guaranteeing the satisfaction of the long-term safety constraint.
Under the assumptions regarding the Lipschitz continuity of the feature mapping function and the existence of a conservative policy, \algo~optimizes a policy while ensuring that there is at least one viable action until the terminal time step.
We theoretically guarantee long-term safety and empirically evaluate the performance of \algo~comparing with several baselines.
Moving forward, it is an interesting direction to improve performance in terms of reward.

\bibliography{ref}

\onecolumn
\appendix
\input{appendix}

\end{document}

%% file: appendix.tex
\section{Proofs}
\label{appendix:proof}

\subsection{Preliminary Lemmas}

\begin{lemma}
\label{lemma:7}
For any states $s, \bar{s} \in \cS$ and (deterministic) policies $\pi, \bar{\pi}$, we have
\begin{alignat*}{2}
    \norm{\phi(s, \pi(s)) - \phi(\bar{s}, \bar{\pi}(\bar{s}))}
    \le L_\phi \cdot d_\cS(s, \bar{s}) + L_\phi \cdot d_\cA(\pi(s), \bar{\pi}(\bar{s})).
\end{alignat*}
\end{lemma}

\smallskip
\begin{proof}
By Assumption~\ref{assumption:lipschitz_feature}, for any states $s, \bar{s} \in \cS$ and policies $\pi, \bar{\pi}$, we have
\begin{alignat*}{2}
    \norm{\phi(s, \pi(s)) - \phi(\bar{s}, \bar{\pi}(\bar{s}))}
    \le &\ L_\phi \cdot d_{\cS\cA}((s, \pi(s)), (\bar{s}, \bar{\pi}(\bar{s}))).
\end{alignat*}
By definition of $d_{\cS \cA}$, we have
\[
    d_{\cS\cA}((s, \pi(s)), (\bar{s}, \bar{\pi}(\bar{s}))) = d_\cS(s, \bar{s}) + d_\cA(\pi(s), \bar{\pi}(\bar{s})).
\]
In the above transformation, we used the assumption that policies $\pi, \bar{\pi}$ are deterministic.
In summary, the following inequality holds:
\begin{alignat*}{2}
    \norm{\phi(s, \pi(s)) - \phi(\bar{s}, \bar{\pi}(\bar{s}))}
    \le &\ L_\phi \cdot d_\cS(s, \bar{s}) + L_\phi \cdot d_\cA(\pi(s), \bar{\pi}(\bar{s})).
\end{alignat*}
\end{proof}

\begin{lemma}
    Suppose, at every time step $t$, the agent's policy $\pi$ satisfies $\norm{\pi(s_t) - \pi^\sharp(s_t)} \le x_t$.
    Then, for any policy $\pi$ and two succeeding states $s_t$ and $s_{t+1} \sim P(\cdot \mid s_t, \pi(s_t))$, we have
    \begin{align}
        \label{eq:diff_phi_t_tp1}
        \norm{\phi(s_{t+1}, \pi(s_{t+1})) - \phi(s_t, \pi(s_t))}
        \le &\ L_\phi \cdot (1 + L_\sharp) \cdot d_\cS(s_t, s_{t+1}) + L_\phi \cdot  (x_t + x_{t+1}).
    \end{align}
\end{lemma}

\smallskip
\begin{proof}
    By Lemma~\ref{lemma:7}, for any policy $\pi$ and two succeeding states $s_t$ and $s_{t+1} \sim P(\cdot \mid s_t, \pi(s_t))$, we have
    \begin{align*}
        \norm{\phi(s_{t+1}, \pi(s_{t+1})) - \phi(s_t, \pi(s_t))}
        \le &\ L_\phi \cdot d_\cS(s_t, s_{t+1}) + L_\phi \cdot d_\cA(\pi(s_t), \pi(s_{t+1})).
    \end{align*}
    By applying the triangle inequality to the second term, we have
    \begin{align*}
        d_\cA(\pi(s_t), \pi(s_{t+1}))
        \le &\ d_\cA(\pi(s_t), \pi^\sharp(s_t)) + d_\cA(\pi^\sharp(s_t), \pi^\sharp(s_{t+1})) + d_\cA(\pi^\sharp(s_{t+1}), \pi(s_{t+1})) \\
        \le &\ x_t + L_\sharp \cdot d_\cS(s_t, s_{t+1}) + x_{t+1}.
    \end{align*}
    In the above transformation, we used Assumption~\ref{assumption:stabilizing}.
    In summary, the following desired inequality can be obtained:
    \begin{align*}
        \norm{\phi(s_{t+1}, \pi(s_{t+1})) - \phi(s_t, \pi(s_t))}
        \le &\ L_\phi \cdot d_\cS(s_t, s_{t+1}) + L_\phi \cdot \left( x_t + L_\sharp \cdot d_\cS(s_t, s_{t+1}) + x_{t+1} \right) \\
        = &\ L_\phi \cdot (1 + L_\sharp) \cdot d_\cS(s_t, s_{t+1}) + L_\phi \cdot  (x_t + x_{t+1}).
    \end{align*}
\end{proof}

\begin{lemma}
    \label{lemma:diff_phi}
    Set $\bar{L}_\sharp \coloneqq L_\sharp + 1$.
    Suppose that, for any time step $t \in [T]$ and state $s_t \in \cS$, a policy $\pi$ takes an action such that $d_\cA(\pi(s_t), \pi^\sharp(s_t)) \le x_t$.
    Then, we have
    \begin{align*}
        &\ \norm{\phi(s_{t+1}, \pi(s_{t+1})) - \phi(s_t, \pi(s_t))} \\
        \le &\ L_\phi \cdot \{\bar{L}_\sharp \cdot \bar{d} + (1 + \eta \bar{L}_\sharp) \cdot x_t + x_{t+1}\}.
    \end{align*}
\end{lemma}

\begin{proof}
    By Assumption~\ref{assumption:stabilizing}, we have
    \begin{align*}
        d_\cS(s_t, s_{t+1}) \le \bar{d} + \eta \cdot x_t.
    \end{align*}
    By applying the above inequality to \eqref{eq:diff_phi_t_tp1}, the following inequality holds:
    \begin{align*}
        \norm{\phi(s_{t+1}, \pi(s_{t+1})) - \phi(s_t, \pi(s_t))}
        \le &\ L_\phi \cdot \{(1 + L_\sharp) \cdot \bar{d} + (1 + \eta + L_\sharp \eta) \cdot x_t + x_{t+1}\} \\
        = &\ L_\phi \cdot \{\bar{L}_\sharp \cdot \bar{d} + (1 + \eta \bar{L}_\sharp) \cdot x_t + x_{t+1}\}.
    \end{align*}
    Then we have the desired lemma.
\end{proof}

\subsection{Proof of Lemma~\ref{lemma:f_t_T}}
\label{proof:A_2}
\begin{lemma}
    \label{lemma:9}
    Suppose that Assumptions~\ref{assumption:lipschitz_feature} and \ref{assumption:stabilizing} hold. Also assume that, at every time step $t$, the agent's policy $\pi$ satisfies $\norm{\pi(s_t) - \pi^\sharp(s_t)} \le x_t$. Then, we have
    \begin{align*}
        \norm{\phi(s_T, \pi(s_T)) - \phi(s_t, \pi(s_t))}
        \le &\ L_\phi \cdot \left\{\bar{L}_\sharp \cdot \bar{d} \cdot (T-t) + (1 + \eta \bar{L}_\sharp) \cdot x_t + (2 + \eta \bar{L}_\sharp) \cdot \sum_{\tau=t+1}^{T-1} x_\tau + x_{T} \right\}.
    \end{align*}
\end{lemma}

\begin{proof}
    By triangle inequality, we have
    \begin{align}
        \norm{\phi(s_T, \pi(s_T)) - \phi(s_t, \pi(s_t))} \le \sum_{i=t}^{T-1} \norm{\phi(s_{i+1}, \pi(s_{i+1})) - \phi(s_i, \pi(s_i))}.
    \end{align}
    By Lemma~\ref{lemma:diff_phi}, the following inequalities hold for all $i \in [t, T-1]$:
    \begin{align}
        \norm{\phi(s_{i+1}, \pi(s_{i+1})) - \phi(s_i, \pi(s_i))}
        = &\ L_\phi \cdot \{\bar{L}_\sharp \cdot \bar{d} + (1 + \eta \bar{L}_\sharp) \cdot x_i + x_{i+1}\}
    \end{align}
    By summing the above inequality up from $i = t, t+1, \ldots, T-1$,
    \begin{align*}
        \norm{\phi(s_T, \pi(s_T)) - \phi(s_t, \pi(s_t))}
        \le &\ L_\phi \cdot \{\bar{L}_\sharp \cdot \bar{d} + (1 + \eta \bar{L}_\sharp) \cdot x_t + x_{t+1}\} \\
        &\ + L_\phi \cdot \{\bar{L}_\sharp \cdot \bar{d} + (1 + \eta \bar{L}_\sharp) \cdot x_{t+1} + x_{t+2}\} \\
        &\ + \cdots \\
        &\ + L_\phi \cdot \{\bar{L}_\sharp \cdot \bar{d} + (1 + \eta \bar{L}_\sharp) \cdot x_{T-1} + x_{T}\} \\
        = &\ L_\phi \cdot \left\{\bar{L}_\sharp \cdot \bar{d} \cdot (T-t) + (1 + \eta \bar{L}_\sharp) \cdot x_t + (2 + \eta \bar{L}_\sharp) \cdot \sum_{\tau=t+1}^{T-1} x_\tau + x_{T} \right\}.
    \end{align*}
    Then we obtained the desired lemma.
\end{proof}

\smallskip
\begin{proof}(of Lemma~\ref{lemma:f_t_T})
    By Assumption~\ref{assumption:linear} and Cauchy–Schwarz inequality, 
    \begin{align*}
        \abs{f^\star(s_T, \pi(s_T)) - f^\star(s_t, \pi(s_t))} &\le \norm{\bm{w}^\star} \cdot \norm{\phi(s_T, \pi(s_T)) - \phi(s_t, \pi(s_t))} \\
        & \le \sqrt{m} \cdot \norm{\phi(s_T, \pi(s_T)) - \phi(s_t, \pi(s_t))}.
    \end{align*}
    By combining the above inequality and Lemma~\ref{lemma:9}, we have
    \begin{equation*}
        \abs{f^\star(s_T, \pi(s_T)) - f^\star(s_t, \pi(s_t))} \le \sqrt{m} \cdot L_\phi \cdot \left\{\bar{L}_\sharp \cdot \bar{d} \cdot (T-t) + (1 + \eta \bar{L}_\sharp) \cdot x_t + (2 + \eta \bar{L}_\sharp) \cdot \sum_{\tau=t+1}^{T-1} x_\tau + x_{T} \right\}.
    \end{equation*}
    Recall the definitions of $L_1$, $L_2$, and $L_3$; that is, $L_1 \coloneqq \sqrt{m} \cdot L_\phi$, $L_2 \coloneqq \bar{L}_\sharp \cdot \bar{d}$, and $L_3 \coloneqq 2 + \eta \bar{L}_\sharp$. Then, the following inequalities holds:
    \begin{equation*}
        \abs{f^\star(s_T, \pi(s_T)) - f^\star(s_t, \pi(s_t))} \le L_1 \cdot \left\{L_2 \cdot (T-t) + (L_3 - 1) \cdot x_t + L_3 \cdot \sum_{\tau=t+1}^{T-1} x_\tau + x_{T} \right\}.
    \end{equation*}
    By definitions of $\bar{t} \coloneqq T - t$ and $X_{t+1}^{T-1} \coloneqq \sum_{\tau=t+1}^{T-1} x_\tau$, we have
    \begin{equation*}
        \abs{f^\star(s_T, \pi(s_T)) - f^\star(s_t, \pi(s_t))} \le L_1 \cdot \left\{L_2 \cdot \bar{t} + (L_3 - 1) \cdot x_t + L_3 \cdot X_{t+1}^{T-1} + x_{T} \right\}.
    \end{equation*}
\end{proof}

\subsection{Proof of Lemma~\ref{lemma:f_1_t}}
\label{proof:A_3}
\begin{lemma}
    Suppose that Assumptions~\ref{assumption:lipschitz_feature} and \ref{assumption:stabilizing} hold. Also assume that, at every time step $t$, the agent's policy $\pi$ satisfies $\norm{\pi(s_t) - \pi^\sharp(s_t)} \le x_t$. Then, we have
    \begin{align*}
        \norm{\phi(s_t, \pi(s_t)) - \phi(s_1, \pi(s_1))}
        \le &\ L_\phi \cdot \left\{\bar{L}_\sharp \cdot \bar{d} \cdot t + (1 + \eta \bar{L}_\sharp) \cdot x_1 + (2 + \eta \bar{L}_\sharp) \cdot \sum_{\tau=2}^{t-1} x_\tau + x_{t} \right\}.
    \end{align*}
\end{lemma}

\smallskip
\begin{proof}
    Similarly to the proof of Lemma~\ref{lemma:f_t_T}, we sum up the inequality \eqref{eq:diff_phi_t_tp1} for $i = 1, 2, \ldots, t-1$ and then obtain
    \begin{align*}
        \norm{\phi(s_t, \pi(s_t)) - \phi(s_1, \pi(s_1))}
        \le &\ L_\phi \cdot \{\bar{L}_\sharp \cdot \bar{d} + (1 + \eta \bar{L}_\sharp) \cdot x_1 + x_{2}\} \\
        &\ + L_\phi \cdot \{\bar{L}_\sharp \cdot \bar{d} + (1 + \eta \bar{L}_\sharp) \cdot x_{2} + x_{3}\} \\
        &\ + \cdots \\
        &\ + L_\phi \cdot \{\bar{L}_\sharp \cdot \bar{d} + (1 + \eta \bar{L}_\sharp) \cdot x_{t-1} + x_{t}\} \\
        = &\ L_\phi \cdot \left\{\bar{L}_\sharp \cdot \bar{d} \cdot t + (1 + \eta \bar{L}_\sharp) \cdot x_1 + (2 + \eta \bar{L}_\sharp) \cdot \sum_{\tau=2}^{t-1} x_\tau + x_t \right\}.
    \end{align*}
    Then we obtained the desired lemma.
\end{proof}

\begin{lemma}
    \label{lemma:11}
    Suppose the policy $\pi$ takes the same action to the conservative policy $\pi^\sharp$ at the initial time step; that is, we set $\pi(s_1) = \pi^\sharp(s_1)$ and then $x_1 = 0$.
    We then have
    \begin{align*}
        \norm{\phi(s_t, \pi_t(s_t)) - \phi(s_1, \pi^\sharp(s_1))}
        \le &\ L_\phi \cdot \left\{\bar{L}_\sharp \cdot \bar{d} \cdot t + (2 + \eta \bar{L}_\sharp) \cdot \sum_{\tau=2}^{t-1} x_\tau + x_{t} \right\}.
    \end{align*}
\end{lemma}

\smallskip
\begin{proof}(of Lemma~\ref{lemma:f_1_t})
    By Assumption~\ref{assumption:linear},
    \begin{align*}
        \abs{f^\star(s_t, \pi(s_t)) - f^\star(s_1, \pi^\sharp(s_1))}
        & \le \norm{\bm{w}^\star} \cdot \norm{\phi(s_t, \pi_t(s_t)) - \phi(s_1, \pi^\sharp(s_1))} \\
        & \le \sqrt{m} \cdot \norm{\phi(s_t, \pi_t(s_t)) - \phi(s_1, \pi^\sharp(s_1))}.
    \end{align*}
    By combining Lemma~\ref{lemma:11} and the aforementioned inequality, we have
    \begin{align*}
        \abs{f^\star(s_t, \pi(s_t)) - f^\star(s_1, \pi^\sharp(s_1))}
        & \le \norm{\bm{w}^\star} \cdot L_\phi \cdot \left\{\bar{L}_\sharp \cdot \bar{d} \cdot t + (2 + \eta \bar{L}_\sharp) \cdot \sum_{\tau=2}^{t-1} x_\tau + x_{t} \right\}.
    \end{align*}
    Define $f^\sharp(s) \coloneqq f^\star(s, \pi^\sharp(s))$ for all $s \in \cS$.
    Based on the definitions of $L_1 \coloneqq \sqrt{m} \cdot L_\phi$, $L_2 \coloneqq \bar{L}_\sharp \cdot \bar{d}$, and $L_3 \coloneqq 2 + \eta \bar{L}_\sharp$, we have
    \begin{align*}
        \abs{f^\star(s_t, \pi(s_t)) - f^\sharp(s_1)}
        \le L_1 \left\{L_2 \, t + L_3 X_{2}^{t-1} + x_{t} \right\}.
    \end{align*}
\end{proof}

\subsection{Proof of Lemma~\ref{lemma:lower_bound}}
\label{appendix:A_5}

\begin{proof}(of Lemma~\ref{lemma:lower_bound})
    By Lemma~\ref{lemma:confidence_bound}, the following inequality holds regarding the lower-bound based on GLM:
    \begin{equation}
        f^\star(s_t, a_t) \ge \ell_{\text{GLM}}(s_t, a_t)
    \end{equation}
    with a probability of at least $1 - \Delta$.
    Also, by Lemma~\ref{lemma:f_1_t}, the following inequality holds (with a probability of $1$) regarding the lower-bound based on the Lipschitz continuity:
    \begin{equation}
        f^\star(s_t, a_t) \ge \ell_{\text{Lipschitz}}(s_t, a_t).
    \end{equation}
    By definition, $\ell(s_t, a_t) \coloneqq \max(\ell_\text{GLM}(s_t, a_t), \ell_\text{Lipschitz}(s_t, a_t))$; then, we have
    \begin{equation}
        f^\star(s_t, a_t) \ge \ell(s_t, a_t)
    \end{equation}
    with a probability of at least $1 - \Delta$.
\end{proof}

\subsection{Proofs of Lemma~\ref{lemma:f_ell_T} and Corollary~\ref{corollary:safety}}

\begin{proof}(of Lemma~\ref{lemma:f_ell_T})
    By Lemma~\ref{lemma:f_t_T},
    \begin{align*}
        f^\star(s_T, \pi(s_T))
        \ge f^\star(s_t, \pi(s_t)) - L_1 \left\{L_2 \bar{t} + (L_3-1) x_t + L_3 X_{t+1}^{T-1} + x_{T} \right\}.
    \end{align*}
    Also, by Lemma~\ref{lemma:lower_bound}, at time step $t$, we have
    \begin{equation*}
        f^\star(s_t, a_t) \ge \ell(s_t, a_t)
    \end{equation*}
    By combining the above two inequalities, we have
    \begin{align*}
        f^\star(s_T, \pi(s_T))
        \ge \ell(s_t, a_t) - L_1 \left\{L_2 \bar{t} + (L_3-1) x_t + L_3 X_{t+1}^{T-1} + x_{T} \right\}.
    \end{align*}
\end{proof}

\begin{proof}(of Corollary~\ref{corollary:safety})
    For a scalar $z \in \mathbb{R}$, suppose that the following inequality holds:
    \[
        \ell(s_t, a_t) - L_1 \left\{L_2 \bar{t} + (L_3-1) x_t + L_3 X_{t+1}^{T-1} + x_{T} \right\} \ge z
    \]
    Then, by Lemma~\ref{lemma:f_ell_T}, the following inequality also immediately holds:
    \[
        f^\star(s_T, \pi(s_T)) \ge z.
    \]  
    In addition, by definitions, we have $L_1 \ge 0$, $L_2 \ge 0$, $L_3 \ge 1$, and $x_\tau \ge 0, \forall \tau \in [T]$.
    Thus, by Lemma~\ref{lemma:lower_bound}, we have
    \[
        f^\star(s_t, \pi(s_t)) \ge \ell(s_t, \pi(s_t)) \ge z.
    \]
    The aforementioned inequalities hold for all $\tau \in [t, T]$; thus, we have
    \[
        f^\star(s_\tau, \pi(s_\tau)) \ge z, \quad \forall \tau \in [t, T].
    \]
\end{proof}

\subsection{Proof of Theorem \ref{theorem:safety}}

\begin{proof}(of Theorem \ref{theorem:safety})
Let $Y_t$ be the event that the following inequality is satisfied:
\begin{equation}
    \label{eq:Y}
    f^\star(s_t, a_t) \ge z.
\end{equation}
Also, Let $Z_t$ be the event that the following inequality holds:
\begin{equation}
    \label{eq:theorem_condition}
    \ell(s_t, a_t) - L_1 \left\{L_2 \bar{t} - x_t + L_3 X_{t}^{T-1} + x_{T} \right\} \ge z.
\end{equation}
When \eqref{eq:theorem_condition} holds, we have
\begin{equation*}
    f^\star(s_\tau, a_\tau) \ge z, \quad \forall \tau \in [t, T],
\end{equation*}
under Assumptions~\ref{assumption:lipschitz_feature} and \ref{assumption:stabilizing} as well as Corollary~\ref{corollary:safety}; hence, the above inequality means the satisfaction of 
\begin{equation}
    \Pr\bigl\{\, Y_t, Y_{t+1}, \ldots, Y_T \mid Z_t \,\bigr\} = 1.
\end{equation}
By Lemma~\ref{lemma:lower_bound}, the lower-bound $\ell$ is probabilistic; that is, $\Pr\bigl\{Z_t\bigr\} \ge 1 - \Delta$; hence,
\begin{align*}
    \Pr\Bigl\{\, Y_t, Y_{t+1}, \ldots, Y_T \,\Bigr\} = 1 - \Delta,
\end{align*}
which means that
\begin{align*}
    \Pr\Bigl\{\, f^\star(s_\tau, a_\tau) \ge z \quad \forall \tau \in [t, T] \,\Bigr\} = 1 - \Delta.
\end{align*}
Finally, the binary safety feedback is stochastic per the GLM; that is, at every time step $\tau \in [t, T]$, safety feedback of $1$ can be obtained with a probability at least $1 - \mu(z)$.
In summary, we have 
\[
    \Pr \Bigl\{ g(s_\tau, a_\tau) = 1 \ \ \forall \tau \in [t, T] \Big\} \ge (1 - \mu(z))^{\bar{t}}.
\]
\end{proof}